\documentclass[A4]{article}
\usepackage{graphicx}
\usepackage{natbib}

\usepackage{nips15submit_e,times}
\nipsfinalcopy
\usepackage{hyperref}
\usepackage{url}
\usepackage{amssymb}
\usepackage{amsmath}
\usepackage{amsthm}
\usepackage{graphicx}
\usepackage{caption,subcaption} 
\usepackage{stmaryrd} 
\usepackage{textcomp}
\usepackage{color}

\usepackage{xspace}
\makeatletter
\DeclareRobustCommand\onedot{\futurelet\@let@token\@onedot}
\def\@onedot{\ifx\@let@token.\else.\null\fi\xspace}
\def\iid{{i.i.d}\onedot}
\def\eg{{e.g}\onedot} 
\def\ie{{i.e}\onedot}

\makeatother

\newcommand{\CHLdrop}[1]{} 

\newcommand{\lebesguecont}{smoothness}
\newcommand{\Lebesguecont}{Smoothness}
\newcommand{\lipschitzhist}{robustness}

\newcommand{\twicediff}{strong robustness}

\newcommand{\Lebesgue}{\lambda}
\newtheorem{lemma}{Lemma}
\newtheorem{theorem}{Theorem}

\newtheorem{definition}{Definition}
\newtheorem{corollary}{Corollary}

\newcommand{\loss}{\ell}

\newcommand{\nR}{\mathbb{R}}

\newcommand{\nF}{\mathcal{F}}
\newcommand{\nH}{\mathcal{H}}
\newcommand{\nI}{\mathcal{I}}
\newcommand{\nJ}{\mathcal{J}}
\newcommand{\nK}{\mathcal{K}}
\newcommand{\nL}{\mathcal{L}}
\newcommand{\nN}{\mathcal{N}}
\newcommand{\nX}{\mathcal{X}}
\newcommand{\nY}{\mathcal{Y}}
\newcommand{\nZ}{\mathcal{Z}}
\newcommand{\natnums}{\mathbb{N}}

\newcommand{\ra}{\rightarrow}

\newcommand{\PP}[1]{\mathbb{P}\left[ #1 \right] }
\newcommand{\PPc}[2]{\mathbb{P}\left[\left.#1\right| #2 \right] }
\newcommand{\EE}[1]{\mathbb{E}\left[ #1 \right] }

\newcommand{\EEnp}[1]{\mathbb{E} #1  }
\newcommand{\EEc}[2]{\mathbb{E}\left[\left.#1\right| #2 \right] }

\newcommand{\Ind}[1]{\mathbb{I}\left[ #1 \right] }

\newcommand{\norm}[1]{\left\lVert #1 \right\rVert}
\newcommand{\abs}[1]{\left| #1 \right|}

\newcommand{\bx}{\mathbf{x}}
\newcommand{\by}{\mathbf{y}}

\newcommand{\bz}{\mathbf{z}}

\newcommand{\floor}[1]{\lfloor #1 \rfloor}

\newcommand{\argmin}{\operatornamewithlimits{argmin}}

\newcommand{\risk}[2]{R( #1 , #2 )}
\newcommand{\hrisk}[2]{\hat{R}( #1 , #2 )}

\newcommand{\hnumerna}{\hat{q}}
\newcommand{\hnumer}[2]{\hnumerna( #1, #2 )}
\newcommand{\numerna}{q}
\newcommand{\numer}[2]{\numerna( #1, #2)}
\newcommand{\hdenumna}{\hat{p}}
\newcommand{\hdenum}[1]{\hdenumna(#1)}
\newcommand{\denumna}{p}
\newcommand{\denum}[1]{\denumna(#1)}
\newcommand{\kernel}[1]{\nK(#1)}
\newcommand{\mrisk}[1]{R_{mar}(h)}
\newcommand{\hmrisk}[1]{\hat{R}_{mar}(h)}

\newcommand{\PPBig}[1]{\mathbb{P}\Big[ #1 \Big] }

\newcommand{\absBig}[1]{\Big| #1 \Big|}
\newcommand{\absbig}[1]{\big| #1 \big|}

\title{Conditional Risk Minimization \\for Stochastic Processes}

\author{
Alexander Zimin \\
IST Austria\\
3400 Klosterneuburg, Austria\\
\texttt{azimin@ist.ac.at} \\
\And
Christoph H. Lampert \\
IST Austria\\
3400 Klosterneuburg, Austria\\
\texttt{chl@ist.ac.at} \\
} 

\begin{document}

\maketitle

\begin{abstract}
We study the task of learning from non-\iid data. In particular, 
we aim at learning predictors that minimize the \emph{conditional risk} 
for a stochastic process, \ie the expected loss of the predictor on the next point conditioned on the set of training samples observed so far.
For non-\iid data, the training set contains information about 
the upcoming samples, so learning with respect to the conditional 
distribution can be expected to yield better predictors than 
one obtains from the classical setting of minimizing the 
\emph{marginal risk}. 
Our main contribution is a practical estimator for the conditional 
risk based on the theory of non-parametric time-series prediction, 
and a finite sample concentration bound that establishes uniform convergence of the estimator to the true conditional risk under 
certain regularity assumptions on the process. 

\end{abstract}

\section{Introduction}
One of the cornerstone assumptions in the analysis of machine 
learning algorithms is that the training examples are independently 
and identically distributed (\iid). 
Interestingly, this assumption is hardly ever fulfilled in practice:
dependencies between examples exist even in the famous 
textbook example of classifying e-mails into \emph{ham} 
or \emph{spam}. For example, when multiple emails are 
exchanged with the same writer, the contents of later 
emails depends on the contents of earlier ones.
For this reason, there is growing interest in the development of 
algorithms that learn from dependent data and still offer generalization 
guarantees similar to the \iid situation. 
%
%

In this work, we are interested in learning algorithms for \emph{stochastic processes}, \ie data sources 
with samples arriving in a sequential manner. 
Traditionally, the generalization performance of learning algorithms 
for stochastic processes is phrased in terms of the \emph{marginal risk}: 
the expected loss for a new data point that is sampled with respect to 
the underlying marginal distribution, regardless of which samples have 
been observed before. 
In this work, we instead are interested in the \emph{conditional risk}, 
\ie the expectation over the loss taken with respect to the 
conditional distribution of the next sample given the samples
observed so far.
For \iid data both notions of risk coincide. 
For dependent data, however, they can differ drastically and 
the conditional risk is the more promising quantity for 
sequential prediction tasks.
Imagine, for example, a self-driving car. 
At any point of time it makes its next decision, e.g. determines if there is a pedestrian in front of the vehicle based on the image from camera.
A typical machine learning approach (based on the marginal risk) would use a single classifier that works well on average.
However, choosing different classifiers at each step such that they are adapted to work well in the current conditions (based on the conditional risk) is clearly beneficial in this case.



There are two main challenges when trying to learn predictors 
of low conditional risk. First, the conditional distribution 
typically changes in each step, so we are trying to learn 
a \emph{moving target}. 
Second, we cannot make use of out-of-the-box empirical risk minimization, 
since that would just lead to predictors of low marginal risk. 

Our main contributions in this work are the following:
\begin{itemize}
\item a non-parametric \textbf{empirical estimator of the conditional risk} 
with finite history, 
\item a proof of \textbf{consistency} under mild assumptions on the process
\item a \textbf{finite sample concentration bound} that, under certain technical 
assumptions, guarantees and quantifies the uniform convergence of the above 
estimator to the true conditional risk.
\end{itemize}
Our results provide the necessary tools to theoretically justify 
and practically perform empirical risk minimization with respect
to the conditional distribution of a stochastic process. 
To our knowledge, our work is the first one providing a consistent algorithm for this problem.


\section{Risk minimization} 


%

We study the problem of \emph{risk minimization}: 
having observed a sequence of examples, $S = \left( \bz_i \right)_{i=1}^N$ 
from a\CHLdrop{stationary $\beta$-mixing} stochastic process, our goal is to select a 
predictor, $h$, of minimal risk from a fixed hypothesis set $\nH$. 
The \emph{risk} is defined as the expected loss for the next observation with 
respect to a given loss function $\loss: \nH \times \nZ \ra [0,1]$.
For example, in a classification setting, one would have $\nZ=\nX\times\nY$, 
where $\nX$ are inputs and $\nY$ are class labels, and solve the task of 
identifying a predictor, $h:\nX\to\nY$, that minimizes the expected $0/1$-loss, 
$\loss(h,\bz)=\Ind{ h(\bx)\neq \by }$, for $\bz=(\bx,\by)$.

Different distribution lead to different definitions of risk.  
The simplest possibility is the \emph{marginal risk},
\begin{align}
\mrisk{h} &= \EE{\loss(h, \bz_{N+1})},
\intertext{which has two desirable properties: it is in fact independent of the actual 
value of $N$ (for the type of the processes that we consider), and under weak conditions 
on the process it can be estimated by a simple average of the losses over the training 
set, \ie the \emph{empirical risk},}
\hmrisk{h} &= \frac{1}{N}\sum_{i=1}^{N}\loss(h, \bz_i).
\end{align}

On the downside, the minimizer of the marginal risk might have low prediction performance on the actual sequence because it tries to generalize across all possible histories, while what we care about in the end is the prediction only for one observed realization.
For an \iid process this would not matter, since any future sample would be 
independent of the past observations anyway. 
For a dependent process, however, the sequence $\bz_1^N$ (where $\bz_i^j$ 
is a shorthand notation for $(\bz_i, \dots, \bz_j)$) might carry valuable 
information about the distribution of $\bz_{N+1}$, see Section~\ref{sec:experiments} 
for an example and numerical simulations.

In this work we study the \emph{conditional risk} for a finite history of length $d$, 

\begin{equation}
\risk{h}{\bar{z}} = \EEc{\loss(h, \bz_{N+1})}{\bz_{N-d+1}^N = \bar{z}}, \label{eq:conditionalrisk}
\end{equation}

for any $\bar{z} \in \nZ^{d}$. 
Our goal is to identify a predictor of minimal conditional risk in the hypothesis 
set, \ie to solve the following optimization problem
\begin{align}\label{eq:main_opt}
&\min_{h\in\nH}\ \risk{h}{\bz_{N-d+1}^N}.
\end{align}

Note that on a practical level this is a more challenging problem than 
marginal risk minimization: the conditional risk depends on the history, 
$\bz_{N-d+1}^N$, so different predictors will be optimal for different histories and time steps. 
However, this comes with the benefit that the resulting predictor is tuned for the actually observed history.

For better understanding let us consider a related problem of time series prediction that can be formulated as a conditional risk minimization. 
One can consider constant predictors and the loss should measure the distance of the prediction from the next value of the process, e.g. for a square loss this would mean minimizing $\EEc{(h-\bz_{N+1})^2}{\bz_{N-d+1}^N}$ over $h\in\nZ$.
Notice that there is a big difference from the standard time series approaches to prediction, where one minimizes a fixed (not changing with time) measure of risk over predictors that can take a finite history into account to make their predictions, which can be written as $\EE{(h(\bz_{N-d+1}^N) - \bz_{N+1})^2}$ with $h: \nZ^d \ra \nZ$.
In this way we choose a fixed function $h$ based on the data and use it henceforth.
In our case, at each step we try to find a (simpler) predictor that minimizes the risk for this particular step, meaning that we can perform optimization over less complex predictors, but we need to recompute it at every step.

\section{Related work} 

While statistical learning theory was first formulated 
for the \iid setting~\citep{vapnik1971uniform}, it 
was soon recognized that extension to non-\iid situations, 
in particular many classes of stochastic processes, were 
possible and useful. 
As in the \iid case, the core of such results is typically
formed by a combination of a \emph{capacity bound} on the 
class of considered predictors and a \emph{law of large numbers} 
argument, that ensures that the empirical average of function 
values converges to a desired expected value.
Combining both, one obtains, for example, that empirical 
risk minimization (ERM) is a successful learning 
strategy.	

Most existing results study the situation of stationary 
stochastic processes, for which the definition of a 
\emph{marginal risk} makes sense. The consistency of 
ERM or similar principles can then be established under 
certain conditions on the dependence structure, for example
for processes that 
are $\alpha$-, $\beta$- or $\phi$-mixing~\citep{YuBin01,karandikar2002rates,steinwart2009fast,zou2009generalization}, 
exchangeable, or \emph{conditionally \iid}~\citep{Berti01,Pestov2010}. 

Asymptotic or distribution dependent results were furthermore obtained even 
for processes that are just ergodic~\citep{adams2010uniform},
and Markov chains with countably infinite state-space~\citep{gamarnik2003extension}.

All of the above works aim to study the minimizer of a long-term 
or marginal risk. 
Actually minimizing the conditional risk has not received much 
attention in the literature, even though some conditional notions 
of risk were noticed and discussed. 
The most popular objective is the conditional risk based on the full history, that is $\EEc{\loss(h, \bz_{N+1})}{\bz_1^N}$.
For example, \cite{Pestov2010} and \cite{Shalizi13} argue in 
favor of minimizing this conditional risk, but focus on exchangeable 
processes, for which the unweighted average over the training samples 
can be used for this purpose. 

The following two papers focus on the variants of the conditioning, but consider the situations where the objective is close to the marginal risk.
\cite{Kuznetsov01} look at the minimization of $\EEc{\loss(h, \bz_{N+s})}{\bz_1^N}$ with integer gap $s$ for the non-stationary 
processes, but the convergence of their bound requires $s$ growing 
as the amount of data grows.
\cite{Mohri02fixed} discuss the conditional
risk based on the full history in the context of generalization guarantees for stable 
algorithms. Their proofs require an assumption that one can freely remove points from the conditioning set without changing the 
distribution.\footnote{Apparently, 
the need for this assumption was realized only after publication of 
the JMLR paper of the same title. Our discussion is based on 
the PDF version of the manuscript from the author homepage, 
dated 10/10/13.}
In our notation this assumption means $\EEc{\loss(h, \bz_{N+1})}{\bz_1^N} = \EEc{\loss(h, \bz_{N+1})}{\bz_1^{N-s}}$ for integer $s$'s.
This again allows the conditional risk to be approximated by the marginal one by separating the point in the loss from the history by an arbitrarily large gap.
In both cases, this makes the problem much easier for mixing processes, since for big values of $s$, $\bz_{N+s}$ is almost independent of $\bz_1^N$.
In contrast to these two works, in our setting the conditional risk is indeed different from the marginal one (see Figure~\ref{fig:distributions} for an example).

\cite{Agarwal01} extend online-to-batch conversion to mixing processes. The authors construct an estimator for the marginal risk, and then show that it can also be used for an average over $m$ future conditional risks. In our notation this corresponds to $\frac{1}{m}\sum_{i=1}^{m}\EEc{\loss(h, \bz_{N+i})}{\bz_1^N}$. Their results are based on the idea of separating the point in the loss and the history by a large enough gap. Similarly to the above papers, the convergence only holds for  $m\ra\infty$, where the average conditional risk converges to the marginal one, while our setting corresponds the case without any gap, $m=1$, with conditioning on a finite history.

\cite{wintenberger2014optimal} introduces a novel online-to-batch conversion technique to show bounds in the regret framework for the cumulative conditional risk, defined as a sum of conditional risks, $\sum_{i=1}^{N}\EEc{\loss(h, \bz_{i+1})}{\bz_1^i}$. Our setting is a harder version of this problem, when we need to minimize each summand separately, not only the whole sum.

The work of \cite{kuznetsov2015learning} is the most related to ours.
They aim to minimize the conditional risk based on the full history by using a weighted empirical average, in the spirit of our estimator.
They provide a generalization bound for fixed, non-random weights and, based on their results, they derive a heuristic procedure for finding the weights without guarantees of convergence.
The main difference of our work is that we provide a data-dependent way to choose weights with the proof of the convergence.

We are not aware of any work that provides a convergent algorithm for the conditional risk based on the full or finite history.
In our work we focus on the finite history and in Section \ref{sec:discussion} we discuss the relation between the two notions.

On a technical level, our work is related to the task of 
\textit{one-step ahead prediction} of time series~\citep{Modha01,modha1998memory,Meir01,alquier2012model}.
The goal of these methods is to reason about the next step of a 
process, though not in order to choose a hypothesis of minimal 
risk, but to predict the value for the next observation itself. 
Our work on empirical conditional risk is inspired by this 
school of thought, in particular on \emph{kernel-based 
nonparametric sequential prediction}~\citep{Biau01}.

\section{Results}\label{sec:proofs}
In this section we present our results together with the assumptions needed for the proofs.

When we want to perform the optimization (\ref{eq:main_opt}) in practice, we do not have access to the conditional distribution of $\bz_{N+1}$. 
Thus, we take the standard route and aim at minimizing an empirical estimator, 
$\hat R$, of the conditional risk.
\begin{align}
&h_S = \argmin_{h\in\nH} \hrisk{h}{\bz_{N-d+1}^N}.
\end{align}
Our first contribution is the definition of a suitable conditional risk estimator when the process takes values in $\nZ \subseteq \nR^k$. 
%
The estimator is based on the notion of a smoothing kernel\footnote{Here and later in this section, as well as in Section~\ref{sec:proofs}, 
we use \emph{kernel} only in the sense of kernel-based non-parametric density estimation, not in the sense of positive definite 
kernel functions from kernel methods.}.


\begin{definition}
A function $\nK: \nR^{kd} \ra \nR_+$ is called a \emph{smoothing kernel} if it satisfies
\begin{enumerate}
\item$ \int_{\nR^{kd}} \nK(\bar{z}) d\bar{z} = 1$,
\item $\nK$ is bounded by $K_1$,
\item $\forall i,j=1,\dots,kd: \int_{\nR^{kd}} \bar{z}_i\kernel{\bar{z}}d\bar{z} = 0,
\quad\text{ and }\quad \int_{\nR^{kd}} \bar{z}_i \bar{z}_j\kernel{\bar{z}}d\bar{z} \leq K_2$,
\item $\nK \text{ is Lipschitz continuous of order } \gamma \text{ with Lipschitz constant } L.$
\end{enumerate}
\end{definition}

A typical example is a squared exponential kernel, $e^{-\gamma\norm{\bar{z}}^2}$, but many other choices are possible.
We can now define our estimator.

\begin{definition}
For a smoothing kernel function, $\nK$, 
and a bandwidth, $b > 0$, we define the empirical conditional risk estimator 
\begin{align}
&\hrisk{h}{\bar{z}} = \frac{\hnumer{h}{\bar{z}}}{\hdenum{\bar{z}}},
\intertext{for}
&\hnumer{h}{\bar{z}} = \frac{1}{nb^{d}}\sum_{i\in \nI}\loss(h, \bz_{i+1})\nK(\,(\bar{z}-\bz_{i-d+1}^i)/b\,)
\intertext{and}
&\hdenum{\bar{z}} = \frac{1}{nb^{d}}\sum_{i\in \nI}\nK(\,(\bar{z}-\bz_{i-d+1}^i)/b\,),
\end{align}
where $\nI = \left\lbrace d, d+1, \dots, N-1 \right\rbrace $ is the index set of samples used
and $n = \abs{\nI}$.
\end{definition}

In words, the estimator is a weighted average loss over the training set, where the weight 
of each sample, $\bz_{i+1}$, is proportional to how similar its history, $\bz_{i-d+1}^i$, is to 
the target history, $\bar z$. 
Similar kernel-based non-parametric estimators have been used successfully in time series 
prediction~\citep{Gyorfi01}. 
Note, however, that risk estimation might be an easier task than that, especially 
for processes of complex objects, since we do not have to predict the actual values 
of $\bz_{N+1}$, but only the loss it causes for a hypothesis $h$. 
In the self-driving car example, compare the pixel-wise prediction of the next image to the prediction of the loss that our classifier causes.

Our main result in this work is the proof that minimizing the above empirical 
conditional risk is a successful learning strategy for finding a minimizer
of the conditional risk.
The following well-known result~\citep{vapnik1998statistical} shows that it  
suffices to focus on uniform deviations of the estimator from the actual risk.
\begin{lemma}
\begin{align}
\risk{h_S}{\bz_{N-d+1}^N} - \inf_{h\in\nH}\risk{h}{\bz_{N-d+1}^N}\ \  \leq \ \ 2\sup_{h\in\nH}\abs{\risk{h}{\bz_{N-d+1}^N} - \hrisk{h}{\bz_{N-d+1}^N}}.
\end{align}
\end{lemma}

As our first result we show the convergence of such uniform deviations to zero.
But before we can make the formal statement, we need to introduce the technical assumptions and a few definitions.

We assume that we observe data from a stationary $\beta$-mixing stochastic process 
$\left\lbrace \bz_i \right\rbrace_{i=1}^\infty$ taking values 
in $\nZ = [0, 1]^k$.
\emph{Stationarity} means that for all $m \geq 1$ the vector $\bz_{1}^{m}$ has the 
same distribution as $\bz_{s+1}^{s+m}$ for all $s \geq 0$.
In order to quantify the dependence between the past and the future of the 
process, we consider mixing coefficients.

\begin{definition}
Let $\sigma(\bz_i^j)$ be a sigma algebra generated by $\bz_i^j$. 
Then the $j$-th $\beta$-mixing coefficient is 
\begin{align}
\beta(j) &= \sup_{t\in\natnums} \EEnp{ \sup_{ A\in\sigma(\bz_{t+j}^\infty)} \abs{\PPc{A}{\bz_{1}^t} - \PP{A}} }.
\end{align}
\end{definition}

A process is called $\beta$-mixing if $\beta(j) \ra 0$ as $j\ra\infty$.
We call a process an exponentially $\beta$-mixing if $\beta(j) \leq c_1e^{-c_2j}$ for some $c_1, c_2 > 0$.
On a high level, a process is mixing if the head of the 
process, $\bz_{1}^t$ and the tail of the process, $\bz_{t+j}^\infty$, 
become as close to independent from each other as wanted when 
they are separated by a large enough gap. 

Many classical stochastic processes are $\beta$-mixing, 
see~\citep{Bradley01} for a detailed survey. 
For example, many finite-state Markov and hidden Markov models 
as well as autoregressive moving average (ARMA) processes 
fulfill $\beta(j)\to 0$ at an exponential rate~\citep{athreya1986mixing},
while certain diffusion processes are $\beta$-mixing at least with 
polynomial rates~\citep{chen2010nonlinearity}. 
Clearly, \iid processes are $\beta$-mixing with $\beta(j)=0$ for all $j$. 

To control the complexity of the hypothesis space we use covering numbers.

\begin{definition}
A set, $V$, of $\nR$-valued functions is a \textbf{$\theta$-cover} 
(with respect to the $\ell_p$-norm) of $\nF \subset \left\lbrace f: \nX \ra \nR \right\rbrace$ on a sample $x_1, \dots, x_n$ if
\begin{equation}
\forall f \in \nF\ \ \exists v\in V 
\quad \Big( \frac{1}{n} \sum_{i=1}^{n} \abs{ f(x_i) - v(x_i) }^p \Big)^{1/p} \leq \theta.
\end{equation}
The \textbf{$\theta$-covering number} of a function class $\nF$ on a given sample $x_1, \dots, x_n$ is
\begin{equation}\nN_p(\theta, \nF, x_1^n ) = \min \left\lbrace \abs{V}: V 
\text{ is an $\theta$-cover w.r.t. $\ell_p$-norm of $\nF$ on $x_1^n$} \right\rbrace.
\end{equation}
The \textbf{maximal $\theta$-covering number} of a function class $\nF$ is
\begin{equation}
\nN_p(\theta, \nF, n ) = \sup_{x_1^n\in\nX^n}\nN_p(\theta, \nF, x_1^n ).
\end{equation}
\end{definition}

\begin{theorem}\label{theorem:consistency}
Assume that
\begin{itemize}
\item there exist $D_0, D_1 > 0$ such that $D_0\Lebesgue(B) \leq \PP{ \bz_{i-d+1}^{i} \in B } \leq D_1 \Lebesgue(B) $ for $B \in \mathcal{B}(\nZ^d)$, where $\Lebesgue$ is the Lebesgue measure on $\nR^{kd}$
\item for every hypothesis $h \in \nH$, the conditional risk $R(h, \bar{z})$ is $L_R$-Lipschitz continuous in $\bar{z}$
\end{itemize}
Then
\begin{equation}
\sup_{h\in\nH}\abs{\risk{h}{\bz_{N-d+1}^N} - \hrisk{h}{\bz_{N-d+1}^N}} \ra 0 \text{ in probability}
\end{equation}
if $b \ra 0 $ as $n \ra \infty$ slowly enough (depending on the covering number of the hypothesis set and the mixing rate of the process). 
The same statement holds almost surely for an exponentially $\beta$-mixing process.
\end{theorem}

We will refer to the first assumption of Theorem \ref{theorem:consistency} as \emph{\lebesguecont} and to the second one as \emph{\lipschitzhist}.
The need for these assumptions is due to the fact that we use nonparametric estimates for the conditional risk, because as it is shown in \citep{Gyorfi06} the ergodicity (which is implied by mixing) itself is not enough to show even $L_1$-consistency of kernel density estimators.
Because of this, additional assumption are required.
\Lebesguecont\ means that the marginal distribution of the process and the Lebesgue measure are mutually absolute continuous. This assumption, for example, satisfied for processes with a density, which is bounded from below away from 0. \cite{caires2005non} argue that it implies some kind of recurrence of the process, that is that for almost every point in the support, the process visits its neighborhood infinitely often.
The usage of local averaging estimation implicitly assumes the continuity of the underlying function, that is the \lipschitzhist\ assumption. The proof would work with weaker, but more technical assumption, however, we stick to this, more natural one.

As a second result we establish the convergence rate of the estimator.

\begin{theorem}\label{theorem:main}
Assume that
\begin{itemize}
\item there exist $D_0, D_1 > 0$ such that $D_0\Lebesgue(B) \leq \PP{ \bz_{i-d+1}^{i} \in B } \leq D_1 \Lebesgue(B) $ for $B \in \mathcal{B}(\nZ^d)$, where $\Lebesgue$ is a Lebesgue measure on $\nR^{kd}$
\item the random vector $\bz_{i-d+1}^i$ has a density $\denumna$. Also, $\numer{h}{\bar{z}} = \risk{h}{\bar{z}}\denum{\bar{z}}$ and $\denumna(\bar{z})$ are both twice continuously differentiable in $\bar{z}$ with second derivatives bounded by $D_2$
\item loss function $\loss$ is $L_{\nH}$-Lipschitz in the first argument
\end{itemize}
Then the following holds for $t_1 = \frac{1}{6} (tD_0 - K_2D_2d^2b^2)$, $t_2 = t_1b^d/(64K_1L_{\nH})$, $t_3 = \big( \frac{3L}{b^{d+\gamma}t_1} \big)^\frac{1}{\gamma}$ and any $\mu, a > 0$ such that $4\mu a d = N$:
\begin{align}
\PPBig{\sup_{h,\bar{z}}\abs{\hrisk{h}{\bar{z}} - \risk{h}{\bar{z}}} > t} 
& \leq \,   32 \left( \frac{\sqrt{kd}t_3}{2} \right)^{kd}  \nN_1(t_2, \nH, n) e^{- \mu t_1^2 b^{2d} / (2048K_1^2)} \\
& + 4 \left( \frac{\sqrt{kd}t_3}{2} \right)^{kd}(\mu-1)\beta(2ad).
\end{align}
\end{theorem}


%

The first assumption in Theorem \ref{theorem:main} is the \lebesguecont\ condition of Theorem \ref{theorem:consistency}.
The second one is a stricter version of \lipschitzhist\ needed to quantify the convergence rate, we will refer to it as \emph{\twicediff}.
The last assumption is a standard way to relate the covering numbers of the induced space $\left\lbrace  \loss(h, \cdot): h \in \nH \right\rbrace $ to $\nH$ (such losses are sometimes called admissible).

As an example, let us consider a case of an exponentially $\beta$-mixing process, e.g. when $\beta(j) \leq e^{-j}$.
For $\nH$ with a finite fat-shattering dimension, $\nN_1(t_2, \nH, n) = poly(n) = poly(N)$, and if we choose $\mu \approx N^{2/3}$, $ 2ad \approx N^{1/3}$ and $b = N^{-\frac{1}{6d}}$, then the bound of Theorem \ref{theorem:main} is approximately $poly(N)(e^{-c_1N^{1/3}} + e^{-c_2N^{1/3}})$ for a fixed $t$.

Before we present the proofs, we introduce some auxiliary results.

\begin{definition}
For a sequence of random variables $\bz_1^n$ and integers $\mu, a$, another random sequence $\tilde{\bz}_1^n$ is called a \textbf{($\mu$,$a$)-independent block copy} of $\bz_1^n$ if $2\mu a = n$ and the blocks $\tilde{\bz}_{ka+1}^{ka+a}$ for $k=0, \dots, 2\mu-1$ are independent and have the same marginal distributions as the corresponding blocks in $\bz_1^n$.
\end{definition}

\begin{lemma}[\cite{YuBin01}, Corollary 2.7]\label{app:beta-mixing}
For a $\beta$-mixing sequence of random variables $\bz_1^n$, let $\tilde{\bz}_1^N$ be its ($\mu$,$a$)-independent block copy.
Then for any measurable function $g$ defined on every second block of $\bz_1^N$ of length $a$ and bounded by $B$, it holds
\begin{equation}
\abs{\EE{g(\bz_1^N)} - \EE{g(\tilde{\bz}_1^N)} } \leq B(\mu-1)\beta(a).
\end{equation}
\end{lemma}

Note that an application of this lemma and the proof of Lemma \ref{lemma:concentration} does not require $2\mu a = n$ to hold exactly. It is possible to work with $\mu$ and $a$ such that $2\mu a < n$ by putting all the remaining points into the last block. However, for the convenience of notations we will write equality.

The proof of Theorems \ref{theorem:consistency} and \ref{theorem:main} relies on Lemma~\ref{lemma:concentration}, a concentration inequality 
that bounds the uniform deviations of functions on blocks of a $\beta$-mixing stochastic process.
This lemma uses a popular independent block technique.
However, it is not a direct application of the previous results (\eg \citep{Mohri01}) as a careful decomposition is required to deal with the fact the summands are defined on overlapping sets of variables.

%
\begin{lemma}\label{lemma:concentration}
Let $\nF \subset \left\lbrace f : \nZ^{d+1} \ra [0, B]\right\rbrace $ be a class of functions
defined on blocks of $d+1$ variables. 
For any integers $\mu, a$, such that $4\mu a d = N$, let $\tilde{\bz}_1^N$ be an ($\mu$,$2ad$)-independent block copy of $\bz_1^N$.
Then
\begin{align}
& \PPBig{\sup_{f\in\nF} \absBig{\frac{1}{n}\sum_{i\in\nI}f(\bz_{i-d+1}^{i+1}) - \EE{f(\bz_{1}^{d+1})} } > t } \\
& \leq 32 \EE{\nN_1(t/64, \nF, \left\lbrace\tilde{\bz}_{i-d+1}^{i+1}, i\in\nI \right\rbrace )} e^{- \mu t^2 / (2048B^2)} + 4(\mu-1)\beta(2ad).
\end{align}
\end{lemma}

\begin{proof}

We start by splitting the index set $\nI = \left\lbrace d, d+1, \dots, N-1 \right\rbrace$ into two sets $\nI_1$ and $\nI_2$, such that $\nI_1 = \left\lbrace i \in \nI : \floor{\frac{i}{d}} \text{ is odd} \right\rbrace $ and $\nI_2 = \left\lbrace i \in \nI : \floor{\frac{i}{d}} \text{ is even} \right\rbrace $.
Then, recalling that $n=N-d-1$ and, hence, $\frac{1}{n} = \frac{N}{N-d-1}\frac{1}{N} \leq \frac{2}{N}$ for $N > 2(d+1),$
\footnote{
if $N \leq 2(d+1)$, then we need $\mu$ and $a$ such that $4\mu a \leq 2(1+\frac{1}{d}) \leq 4$, which holds only for $\mu = a = 1$ and the bound of the lemma is trivial in this case.
}
\begin{align}
& \PPBig{\sup_{f\in\nF} \absBig{\frac{1}{n}\sum_{i\in\nI}f(\bz_{i-d+1}^{i+1}) - \EE{f(\bz_{1}^{d+1})} } > t } \\ & \leq \sum_{j=1}^{2}\PPBig{\sup_{f\in\nF} \absBig{\frac{1}{N}\sum_{i\in\nI_j}f(\bz_{i-d+1}^{i+1}) - \EE{f(\bz_{1}^{d+1})} } > t/4 }.
\end{align}

Both summands can be bounded in the same way, so we focus on the first one.
We are going to use the independent block technique due to \cite{YuBin01}.
For this we choose $\mu$ and $a$ such that $4\mu a d  = N$.
We will split the sample $\bz_1^N$ into $2\mu$ blocks of $2ad$ consecutive points (Note that is a different splitting than the one above, since here we split the variables themselves).
Thanks to the above step, there is no function that takes variables from different blocks.
Let $\nJ^e = \left\lbrace y_1, \dots, y_\mu \right\rbrace $, where $y_k = ((2k-2)2ad+1, \dots, (2k-2)2ad+2ad)$, and $\nJ^o = \left\lbrace \bar{y}_1, \dots, \bar{y}_\mu \right\rbrace $, where $\bar{y}_k = ((2k-1)2ad+1, \dots, (2k-1)2ad+2ad)$.
In this way, the blocks within $\nJ^e$ and $\nJ^o$ are separated by gaps of $2ad$ points.
Note that thanks to the first splitting, we can rewrite
\begin{align}
\frac{1}{N}\sum_{i\in\nI_1}f(\bz_{i-d+1}^{i+1}) & = \frac{1}{N}\sum_{i\in\nI_1\cap\nJ^e}f(\bz_{i-d+1}^{i+1}) + \frac{1}{N}\sum_{i\in\nI_1\cap\nJ^o}f(\bz_{i-d+1}^{i+1}) \\
& = \frac{1}{2\mu}\sum_{j=1}^{\mu}\frac{1}{2ad} \sum_{i\in\nI_1\cap y_j}f(\bz_{i-d+1}^{i+1}) + \frac{1}{2\mu}\sum_{j=1}^{\mu}\frac{1}{2ad} \sum_{i\in\nI_1\cap\bar{y}_j}f(\bz_{i-d+1}^{i+1}) \\
& = \frac{1}{2\mu}\sum_{j=1}^{\mu} F_f(\bz_{y_j}) + \frac{1}{2\mu}\sum_{j=1}^{\mu} F_f(\bz_{\bar{y}_j}),
\end{align}
where we defined composite hypotheses $F_f(\bz_y) = \frac{1}{2ad} \sum_{i\in\nI_1\cap y}f(\bz_{i-d+1}^{i+1})$ with $\bz_{y} = \left\lbrace \bz_i \right\rbrace_{i\in y} $.
Then
\begin{align}
& \PPBig{\sup_{f\in\nF} \absBig{\frac{1}{N}\sum_{i\in\nI_1}f(\bz_{i-d+1}^{i+1}) - \EE{f(\bz_{1}^{d+1})} } > t/4 }\\
 & \leq 2 \PPBig{\sup_{f\in\nF} \absBig{\frac{1}{\mu}\sum_{j=1}^{\mu} F_f(\bz_{y_j}) - \EE{F_f(\bz_{y_1})} } > t/4 }.
\end{align}
Let $\tilde{\bz}_{y_1}, \dots, \tilde{\bz}_{y_\mu}$ be the random variables having the same marginal distributions as $\bz_{y_j}$'s, but drawn independently from each other.
Using the fact that probability of an event is an expectation of the indicator of the same event, we can apply Lemma \ref{app:beta-mixing} to obtain
\begin{align}
& \PPBig{\sup_{f\in\nF} \absBig{\frac{1}{\mu}\sum_{j=1}^{\mu} F_f(\bz_{y_j}) - \EE{F_f(\bz_{y_1})} } > t/4 } \\ 
&  \leq \PPBig{\sup_{f\in\nF} \absBig{\frac{1}{\mu}\sum_{j=1}^{\mu} F_f(\tilde{\bz}_{y_j}) - \EE{F_f(\tilde{\bz}_{y_1})} } > t/4 } + (\mu-1)\beta(2ad).
\end{align}
The first term can be bounded using the standard techniques for uniform laws of large numbers for \iid random variables.
We are going to use the bound in terms of covering numbers.
Following the standard proof, \eg \citep[Theorem 9.1]{Gyorfi01}, we obtain
\begin{align}
\PPBig{\sup_{f\in\nF} \absBig{&\frac{1}{\mu}\sum_{j=1}^{\mu} F_f(\tilde{\bz}_{y_j}) - \EE{F_f(\tilde{\bz}_{y_1})} } > t/4 } \\ &\leq 8 \EE{\nN_1(t/32, F(\nF), \left\lbrace\tilde{\bz}_{y_i} \right\rbrace_{j=1}^\mu) }e^{-\mu t^2/(2048B^2)},
\end{align}
where $F(\nF)$ is a class of composite hypotheses.
The only thing that is left is to connect the covering number of $F(\nF)$ to the covering number of $\nF$.
This follows from the fact that for any $f, g \in \nF$ and any fixed blocks $y_1, \dots, y_\mu$ on $z_1^N$:
\begin{align}
\frac{1}{\mu}\sum_{j=1}^\mu \abs{F_f(z_{y_i}) - F_g(z_{y_i})} & \leq \frac{2}{N}\sum_{j=1}^\mu \sum_{i\in\nI_1\cap y_i} \abs{f(z_{i-d+1}^{i+1}) - g(z_{i-d+1}^{i+1})} \\
& \leq \frac{2}{n}\sum_{i\in\nI} \abs{f(z_{i-d+1}^{i+1}) - g(z_{i-d+1}^{i+1})}.
\end{align}

\end{proof}

\begin{corollary}\label{cor:concentration}
For a fixed $\bar{y} \in \nZ^{d}$ let 
\begin{equation}
\nF = \left\lbrace f(\bar{z}) = \frac{1}{b^{d}}\loss(h, \bar{z}_{d+1})\kernel{(\bar{y} - \bar{z}_1^d)/b}  : h \in \nH \right\rbrace \subset \{f:\nZ^{d+1} \ra \nR_+\}.
\end{equation}
Assume that loss function $\loss$ is $L_{\nH}$-Lipschitz in the first argument.
Then under the conditions of Lemma \ref{lemma:concentration},
the following holds for $\tilde{t} = tb^d/(64K_1L_{\nH})$
\begin{align}
\PPBig{\sup_{f\in\nF} &\absBig{\frac{1}{n}\sum_{i\in\nI}f(\bz_{i-d+1}^{i+1}) - \EE{f(\bz_{1}^{d+1})} } > t } \\
& \leq 32 \nN_1(\tilde{t}, \nH, n ) e^{- \mu t^2 b^{2d} / (2048K_1^2)} + 4(\mu-1)\beta(2ad).
\end{align}
\end{corollary}

\begin{proof}
The corollary follows from the proof of Lemma \ref{lemma:concentration} by a further upper bound on the covering number.
First, for any two $f, f' \in \nF$ on fixed $z_1^n$
\begin{align}
\frac{1}{n}\sum_{i\in\nI}&\abs{ f(z_{i-d+1}^{i+1}) - f'(z_{i-d+1}^{i+1}) }\\ & =  \frac{1}{nb^d}\sum_{i\in\nI}\abs{ \nK((\bar{y}-z_{i-d+1}^i)/b)(\loss(h, z_{i+1}) - \loss(h', z_{i+1})) } \\
& \leq\! \frac{K_1}{nb^d}\sum_{i\in\nI} \abs{ \loss(h, z_{i+1}) - \loss(h', z_{i+1}) }.
\end{align}
Hence, $\nN_1(\varepsilon, \nF, \left\lbrace\tilde{\bz}_{i-d+1}^{i+1}, i\in\nI \right\rbrace) \leq \nN_1(\varepsilon b^d/K_1, \nL(\nH), \left\lbrace\tilde{\bz}_{i+1}, i\in\nI \right\rbrace)$, where $\nL(\nH) = \left\lbrace f(z) = \loss(h, z) : h \in\ \nH \right\rbrace$.
Second, by the Lipschitz property of the loss function we get: $\nN_1(\varepsilon b^d/K_1, \nL(\nH), \left\lbrace\tilde{\bz}_{i+1}, i\in\nI \right\rbrace) \leq \nN_1(\varepsilon b^d/(K_1 L_\nH), \nH, \left\lbrace\tilde{\bz}_{i+1}, i\in\nI \right\rbrace)$.
\end{proof}

Now we are ready to prove Theorem \ref{theorem:consistency}.


The proof is based on the argument of \cite{collomb1984proprietes} with appropriate modifications to achieve the uniform convergence over hypotheses.
\begin{proof}[Theorem \ref{theorem:consistency}]
We start by reducing the problem to the supremum over the histories.
\CHLdrop{
Denote the $\sup_{h\in\nH}\abs{\risk{h}{\bar{z}} - \hrisk{h}{\bar{z}}}$ by $\Psi(\bar{z})$. Then for every $\varepsilon > 0$:
\begin{align}
\PP{\Psi(\bz_{N-d+1}^N) > t} & \leq \PPc{\Psi(\bz_{N-d+1}^N) > t}{\bz_{N-d+1}^N \in K_\varepsilon}\PP{\bz_{N-d+1}^N \in K_\varepsilon} + \varepsilon \\
& \leq \PPc{\sup_{\bar{z}\in K_\varepsilon} \Psi(\bar{z}) > t}{\bz_{N-d+1}^N \in K_\varepsilon}\PP{\bz_{N-d+1}^N \in K_\varepsilon} + \varepsilon \\
& = \PP{\sup_{\bar{z}\in K_\varepsilon} \Psi(\bar{z}) > t \ \text{and} \ \bz_{N-d+1}^N \in K_\varepsilon} + \varepsilon \\
& \leq \PP{\sup_{\bar{z}\in K_\varepsilon} \Psi(\bar{z}) > t} + \varepsilon
\end{align}
}
Then we make the following decomposition
\begin{align}
\sup_{h, \bar{z}}\abs{\risk{h}{\bar{z}} - \hrisk{h}{\bar{z}}} \leq  \frac{1}{\inf_{\bar{z}} \EE{\hdenum{\bar{z}}}}( T_1 + T_2 + T_3 ),
\end{align}
where 
\begin{align}
T_1 & = \sup_{h, \bar{z}}\abs{\hrisk{h}{\bar{z}}(\EE{\hdenum{\bar{z}}} - \hdenum{\bar{z}})}, \\
T_2 & = \sup_{h, \bar{z}}\abs{\hnumer{h}{\bar{z}} - \EE{\hnumer{h}{\bar{z}}}}, \\
T_3 & = \sup_{h, \bar{z}}\abs{\EE{\hnumer{h}{\bar{z}}} - \risk{h}{\bar{z}}\EE{\hdenum{\bar{z}}}}.
\end{align}
First, note that by the \lebesguecont\ assumption $\EE{\hdenum{\bar{z}}} \geq D_0 \int\nK(\bar{u})d\bar{u} = D_0$.
A minor modification of Lemma 5 from \citep{collomb1984proprietes} coupled with \lipschitzhist\ gives us the convergence of $T_3$.
Next, note that, by Lemma \ref{lemma:densitycovering}, \ $\PP{T_i > t} \leq \nN\PP{\tilde{T}_i > t'}$ for $i=1,2$, where
\begin{align}
\tilde{T}_1 & = \sup_{h}\abs{\EE{\hdenum{\bar{z}}} - \hdenum{\bar{z}}}, \\
\tilde{T}_2 & = \sup_{h}\abs{\hnumer{h}{\bar{z}} - \EE{\hnumer{h}{\bar{z}}}}.
\end{align}
and $\nN$ is a covering of a $kd$-dimensional hypercube with an appropriate ball width.
Now, by Lemma \ref{lemma:concentration}, we get a bound on $\tilde{T}_1$ and $\tilde{T}_2$ and obtain the convergence in probability.

For exponentially $\beta$-mixing processes, the same Lemma gives an exponential bound on $\tilde{T}_1$ and $\tilde{T}_2$ and hence we get the almost sure convergence of $T_1$ and $T_2$ to 0 (by the Borel-Cantelli lemma).
\end{proof}

The proof of the convergence rate requires a bit different decomposition than in Theorem \ref{theorem:consistency} and more delicate treatment of the terms using the additional assumptions. 
We introduce two further lemmas: Lemma~\ref{lemma:firstargument} shows how to 
express the concentration of $\hrisk{h}{\bar{z}}$ in terms of the concentration 
of $\hnumer{h}{\bar{z}}$ and $\hdenum{\bar{z}}$.
Lemma~\ref{lemma:densitycovering} uses covers to eliminate the supremum over $\bar{z}$.
\begin{lemma}\label{lemma:firstargument}
Assume \lebesguecont\ and \twicediff.
%
%
Then, for $R=\risk{h}{\bar{z}}$, $\hat R=\hrisk{h}{\bar{z}}$, 
$\hat q=\hnumer{h}{\bar{z}}$, $\hat p=\hdenum{\bar{z}}$, and with $t_1 = \frac{1}{2}(tD_0 - K_2D_2d^2b^2)$,
\begin{eqnarray}\label{eq:main-split}
\PPBig{\sup_{h,\bar{z}}\absbig{\hat R - R} > t} &\leq& \PPBig{\sup_{h,\bar{z}}\absbig{\hat q - \EE{\hat q}} > t_1} + \PPBig{\sup_{\bar{z}} \absbig{\hat p - \EE{\hat p}} > t_1}.
\end{eqnarray}
\end{lemma}

\begin{proof}
We start with the following decomposition
\begin{align}
\hrisk{h}{\bar{z}} - \risk{h}{\bar{z}} &= 
\frac{ \hrisk{h}{\bar{z}}( \EE{\hdenum{\bar{z}}} - \hdenum{\bar{z}}) }{\EE{\hdenum{\bar{z}}}} + \frac{\hnumer{h}{\bar{z}} - \numer{h}{\bar{z}}}{\EE{\hdenum{\bar{z}}}} \\& + \frac{\risk{h}{\bar{z}}(\denum{\bar{z}} - \EE{\hdenum{\bar{z}}})}{\EE{\hdenum{\bar{z}}}}.
\end{align}
Using the fact that \lebesguecont\ implies $\EE{\hdenum{\bar{z}}} \geq D_0 \int \nK(\bar{z})d\bar{z} = D_0$ and that $\hrisk{h}{\bar{z}}$ and $\risk{h}{\bar{z}}$ are both upper bounded by 1, we can bound the left hand side of (\ref{eq:main-split}) by
\begin{align}
& \PP{\sup_{h,\bar{z}}\abs{\hnumer{h}{\bar{z}} - \numer{h}{\bar{z}}} > \frac{1}{2}tD_0} \\
& + \PP{\sup_{\bar{z}}\abs{\hdenum{\bar{z}} - \EE{\hdenum{\bar{z}}}} + \sup_{\bar{z}}\abs{\denum{\bar{z}} - \EE{\hdenum{\bar{z}}}} > \frac{1}{2}tD_0}.
\end{align}
The statement of the lemma will be proven if we show that $\abs{\numer{h}{\bar{z}} - \EE{\hnumer{h}{\bar{z}}} }$ and $\abs{ \denum{\bar{z}} - \EE{\hdenum{\bar{z}}} }$ are upper bounded by $\frac{1}{2}K_2D_2d^2b^2$.
We demostrate this only for $\numerna$.
Using the stationarity,
\begin{align}
\EE{\hnumer{h}{\bar{z}}} = \frac{1}{b^{d+1}}\int \kernel{(\bar{z}-\bar{u})/b} \numer{h}{\bar{u}} d\bar{u} = \int \kernel{\bar{u}} \numer{h}{\bar{z} - b\bar{u}}d\bar{u}.
\end{align}
Now we apply the Taylor expansion:
\begin{align}
\numer{h}{\bar{z}} - \EE{\hnumer{h}{\bar{z}}} & = - b \sum_{i=1}^{d} \frac{\partial}{\partial \bar{z}_i}\numer{h}{\bar{z}} \int \bar{u}_i \nK(\bar{u}) d\bar{u} \\&\quad + \frac{b^2}{2}\sum_{i,j =1}^{d}\frac{\partial^2}{\partial \bar{z}_i \partial \bar{z}_j}\numer{h}{\bar{\xi}} \int \bar{u}_i \bar{u}_j \nK(\bar{u}) d\bar{u}
\intertext{and, invoking the assumptions on the kernel,}
\Big|{\numer{h}{\bar{z}} - \EE{\hnumer{h}{\bar{z}}} }\Big| &\leq \frac{1}{2} D_2 K_2 d^2b^2.
\end{align}
\end{proof}

\begin{lemma}\label{lemma:densitycovering}
Let $\nN_\tau$ be an $\tau$-covering number of a $kd$-dimensional hypercube (in $\ell_2$-norm) with $\tau =  \left( \frac{b^{d+\gamma}t_1}{3L} \right)^\frac{1}{\gamma}$, then
\begin{eqnarray}
	\PPBig{\sup_{h,\bar{z}}\abs{\hnumer{h}{\bar{z}} - \EE{\hnumer{h}{\bar{z}}}} > t_1} &\leq& \nN_\tau \PPBig{\sup_{h}\abs{\hnumer{h}{\bar{z}} - \EE{\hnumer{h}{\bar{z}}}} > \frac{t_1}{3}},
\\
\PP{\sup_{\bar{z}}\abs{\hdenum{\bar{z}} - \EE{\hdenum{\bar{z}}}} > t_1} &\leq& \nN_\tau \PP{\abs{\hdenum{\bar{z}} - \EE{\hdenum{\bar{z}}}} > \frac{t_1}{3}}.
\end{eqnarray}
\end{lemma}

\begin{proof}
We again prove the statement only for $\numerna$, while the argument for $\denumna$ goes along the same lines.
Let us consider $V$, a fixed $\tau$-covering of $\nZ^d$ with $\tau$ to be set later.
We denote by $v(\bar{z})$ the closest element of the covering to $\bar{z}$.
Then we have
\begin{align}
\sup_{h,\bar{z}}\abs{\hnumer{h}{\bar{z}} - \EE{\hnumer{h}{\bar{z}}}} & \leq \sup_{h,\bar{z}}\abs{\hnumer{h}{\bar{z}} - \hnumer{h}{v(\bar{z})}} \\& + \sup_{h,\bar{z}}\abs{\hnumer{h}{v(\bar{z})} - \EE{\hnumer{h}{v(\bar{z})}}} \\
& + \sup_{h,\bar{z}}\abs{\EE{\hnumer{h}{v(\bar{z})}} - \EE{\hnumer{h}{\bar{z}}}}.
\end{align}
Using the Lipschitz property of the kernel, we can bound $\hnumer{h}{\bar{z}} - \hnumer{h}{v(\bar{z})}$ by
\begin{align}
&\frac{1}{nb^{d}} \sum_{i\in\nI}\loss(h, \bz_{i+1})\left( \kernel{(\bar{z}-\bz_{i-d+1}^i)/b} - \kernel{(v(\bar{z})-\bz_{i-d+1}^i)/b} \right) \\
& \leq \frac{L}{nb^{d+\gamma}}\sum_{i\in\nI}\loss(h, \bz_{i+1})\norm{\bar{z} - v(\bar{z})}_2^\gamma \leq \frac{L\tau^\gamma}{b^{d+\gamma}}.
\end{align}
Now, setting $\tau = \left( \frac{b^{d+\gamma}t_1}{3L} \right)^\frac{1}{\gamma}$, we can ensure that
\begin{equation}
\sup_{h,\bar{z}}\abs{\hnumer{h}{\bar{z}} - \EE{\hnumer{h}{\bar{z}}}} \leq \sup_{h,v\in V}\abs{\hnumer{h}{v} - \EE{\hnumer{h}{v}}} + \frac{2}{3}t_1,
\end{equation}
which lead us to the statement of the lemma.
\end{proof}

\CHLdrop{CHL: Move to supplemental? It is worth noting that the covering number can be further upper bounded by $\frac{\text{diam}(G)\sqrt{d}}{(2\tau)^d}$, e.g.\ see \citep{London01}, while we do not include this bound into the statement to avoid overloading of the formulas.}

%
\begin{proof}[Theorem \ref{theorem:main}]
The proof is a combination of Lemmas \ref{lemma:firstargument} and \ref{lemma:densitycovering}, followed by Corollary \ref{cor:concentration}.
To obtain the final result we also bound the $\tau$-covering number of $kd$-dimensional hypercube by $\left( \frac{\sqrt{kd}}{2\tau} \right)^{kd}$.
Note that, technically, to have the concentration for $\hdenumna$ we do not need to take the last step in Lemma \ref{lemma:concentration} with coverings, however, to unify the final statement, we include the covering term in the bound.
\end{proof}

\section{Simulations}\label{sec:experiments}
This section illustrates
the problem setting and the proposed estimator of the conditional 
risk in a synthetic setting that is easy to analyze but difficult 
to learn for other techniques. 
Our goal is to highlight the differences between marginal and conditional risks and between conditioning on the finite history and full history.

Let $\nZ=\nX\times\nY$ with $\nX = [0,10]\times[0,10]\in\nR^2$ 
and $\nY= \{\pm 1\}$. 
%
We generate data using a time-homogeneous hidden Markov process with four 
latent states (Figure~\ref{subfig:process}). Each state, $i$, is associated 
with an emission probability distribution, $\mu_i(x,y)$, that is uniform in 
$x$ and deterministic in $y$, with $y=\operatorname{sign} f_i(x)$ for an 
affine function $f_i$. 
At any step, $i$, we observe a sample $z_i=(x_i,y_i)$ from the distribution  
associated with the current latent state $s_i$. 
Figure~\ref{subfig:sequence} depicts the situation for $N=1000$. 

Drawn without order, the empirical distribution of the samples $z_1^{N}$ 
resembles the limit distribution of the hidden Markov chain~(Figure~\ref{subfig:stationary}), 
which is also the marginal distribution of the next sample, $p(z_{N+1})$. 
Using the dependencies in the sequence, however, we can obtain a more informed 
estimate, $p(z_{N+1}|z_1^{N})$ or its finite history counterparts, $p(z_{N+1}|z_{N-d}^{N})$ 
for any history length $d$. 
Figures~\ref{subfig:conditional-h1} to \ref{subfig:conditional-full} visualize these distribution:
already with a short history length, the conditional distribution is easier learnable 
(has a lower Bayes risk) than the marginal one. 
Thus, identifying a good classifier for the conditional distribution at each step, \ie 
minimizing the conditional risk, can lead to an overall lower error rate than finding a 
single predictor that is optimal for the limit distribution, \ie minimizing the marginal 
risk. 

\begin{figure*}[t]\centering
	\begin{subfigure}[t]{.34\textwidth}\centering
		\includegraphics[width=\textwidth]{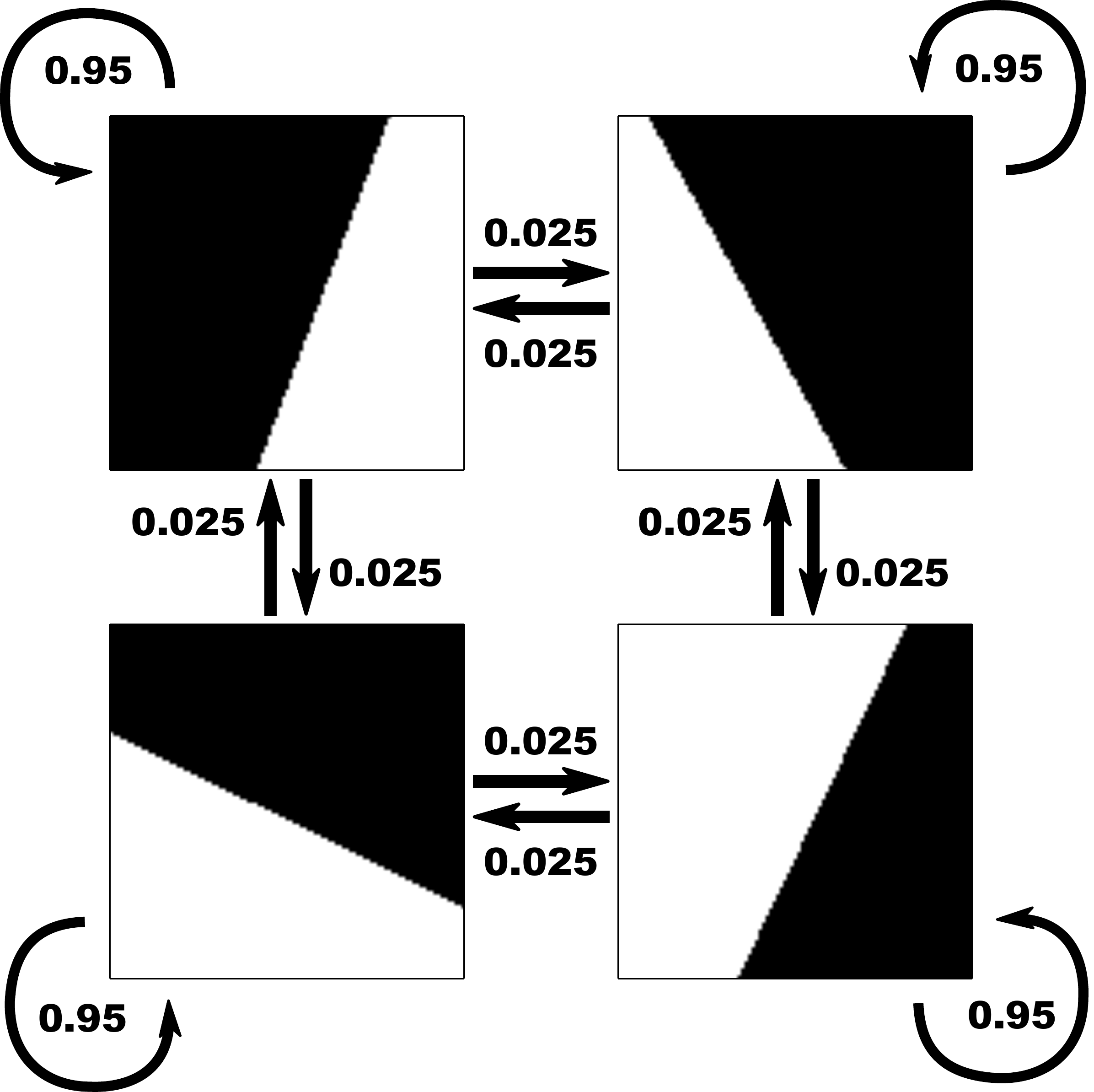}
		\caption{Generating Markov Process}\label{subfig:process}
	\end{subfigure}
	\begin{subfigure}[t]{.31\textwidth}\centering
		\raisebox{.02\textwidth}{\includegraphics[width=\textwidth]{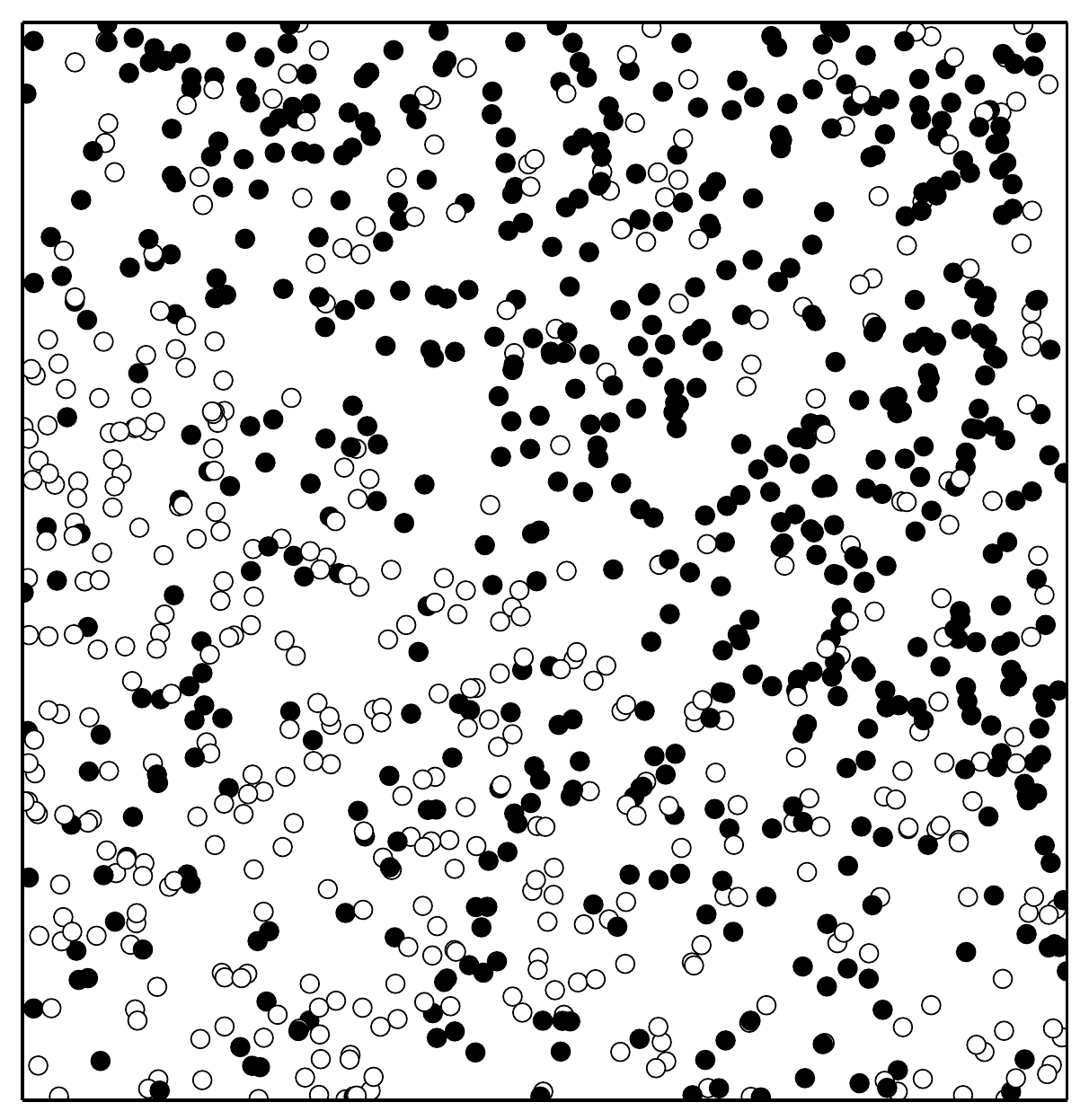}}
		\caption{Sample sequence, $z_1^{1000}$}\label{subfig:sequence}
	\end{subfigure}
	\begin{subfigure}[t]{.31\textwidth}\centering
        \raisebox{.02\textwidth}{\includegraphics[width=\textwidth]{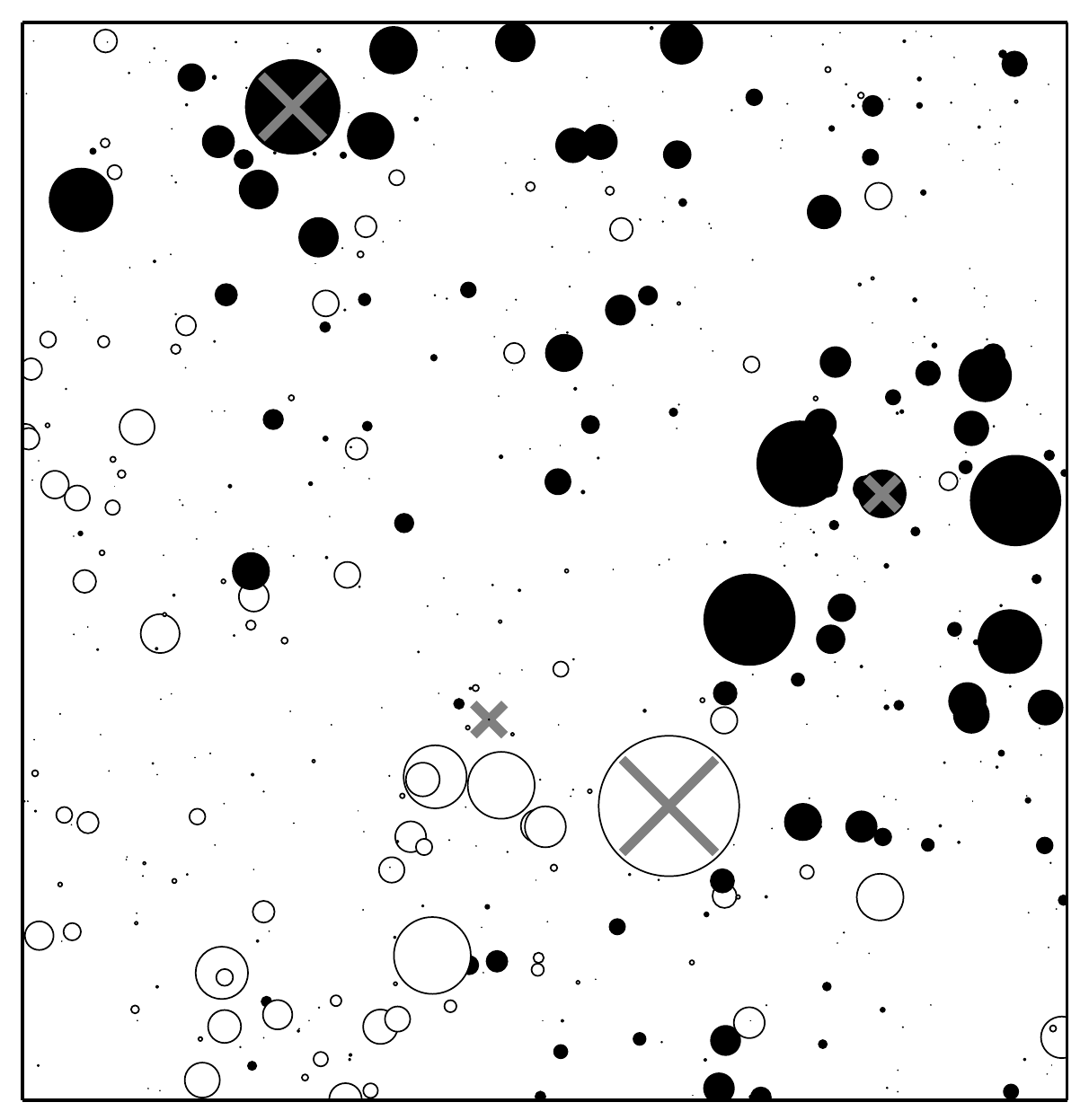}} 
        \caption{Contribution of each point to the conditional risk estimate}\label{subfig:weights} 
\end{subfigure}
\end{figure*}

\begin{figure*}[t]\centering
    \begin{subfigure}[t]{.19\textwidth}\centering
		\includegraphics[width=\textwidth]{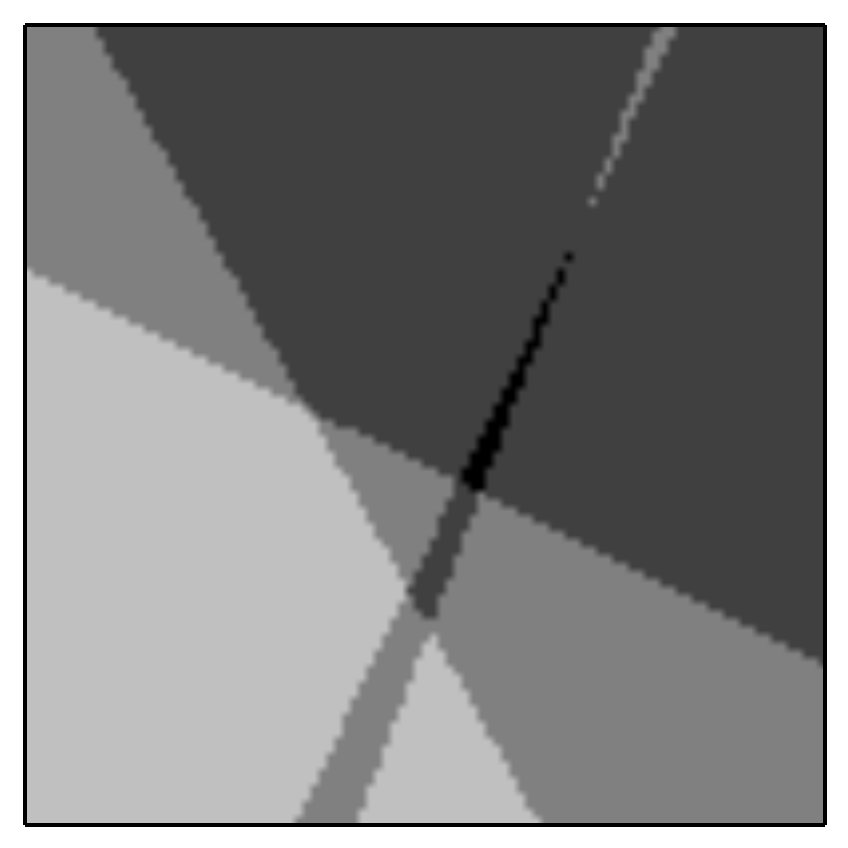}
		\caption{$p(z_{N+1})$} \label{subfig:stationary}
	\end{subfigure}
	\begin{subfigure}[t]{.19\textwidth}\centering
		\includegraphics[width=\textwidth]{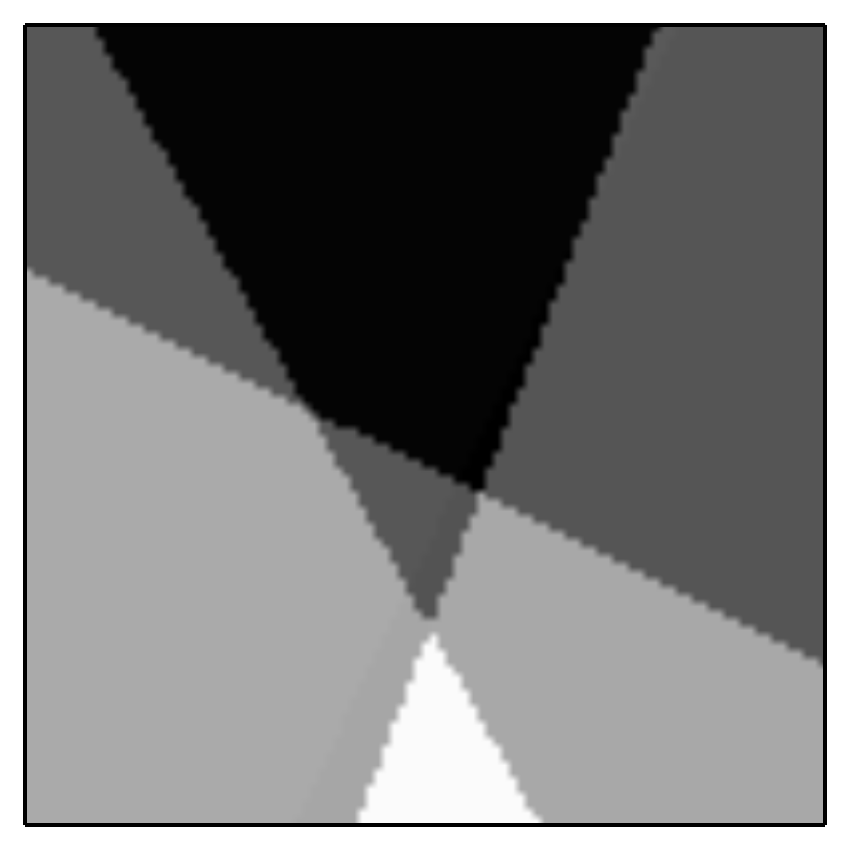}
		\caption{$p(z_{n+1}|z_{n-1}^{n})$}\label{subfig:conditional-h1} 
	\end{subfigure}
	\begin{subfigure}[t]{.19\textwidth}\centering
		\includegraphics[width=\textwidth]{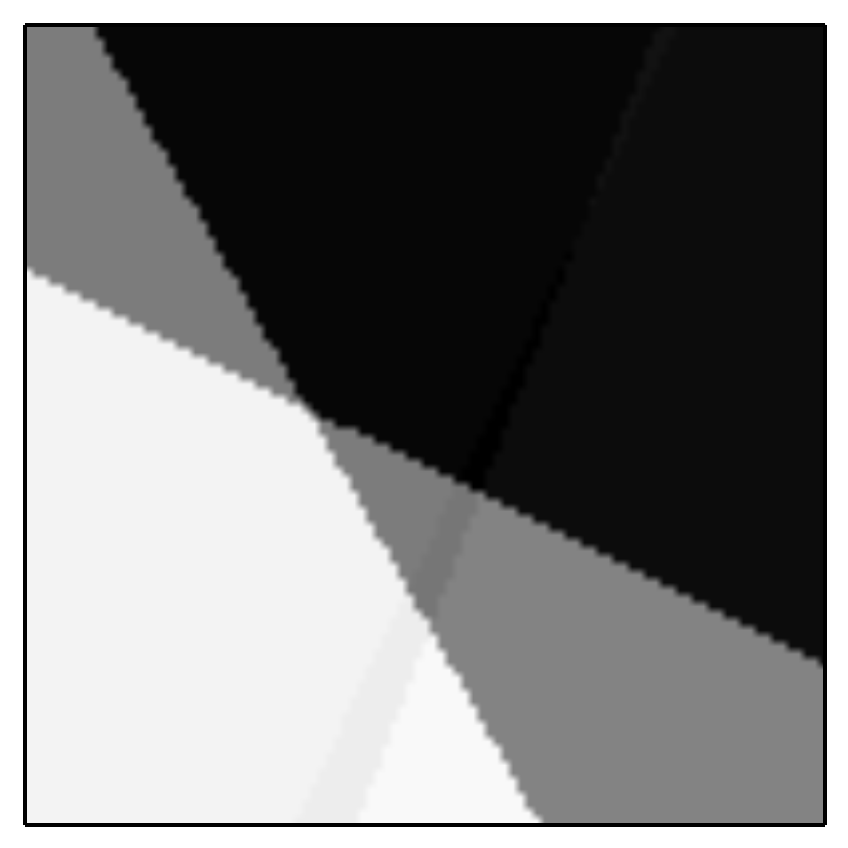}
		\caption{$p(z_{n+1}|z_{n-2}^{n})$}\label{subfig:conditional-h2} 
	\end{subfigure}
	\begin{subfigure}[t]{.19\textwidth}\centering
		\includegraphics[width=\textwidth]{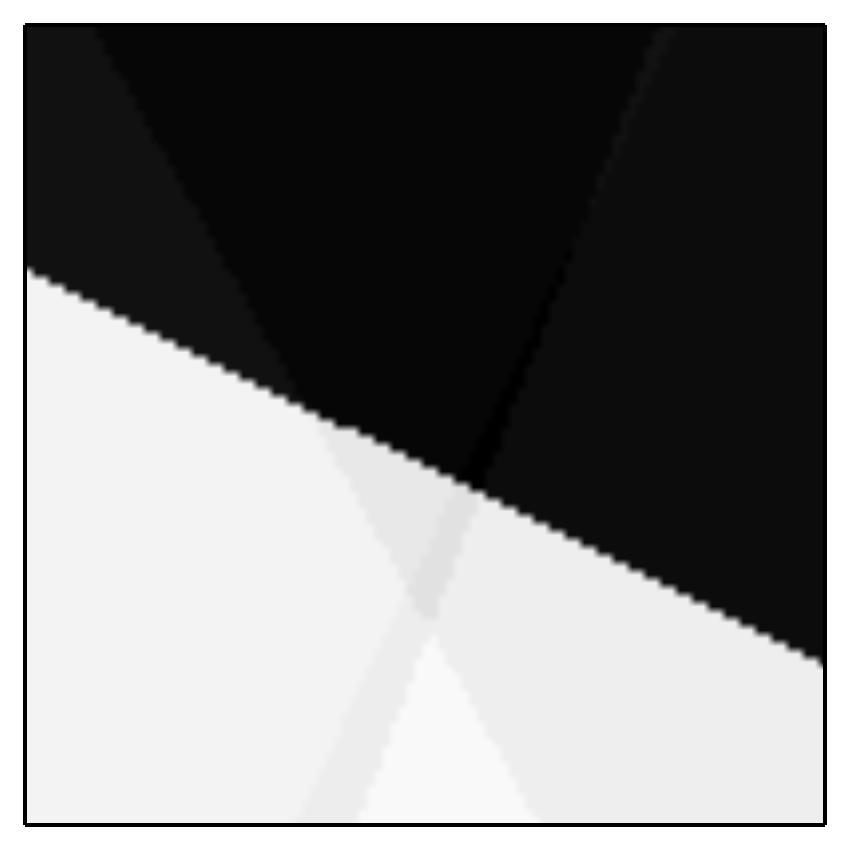}
		\caption{$p(z_{n+1}|z_{n-4}^{n})$}\label{subfig:conditional-h4}
	\end{subfigure}
	\begin{subfigure}[t]{.19\textwidth}\centering
		\includegraphics[width=\textwidth]{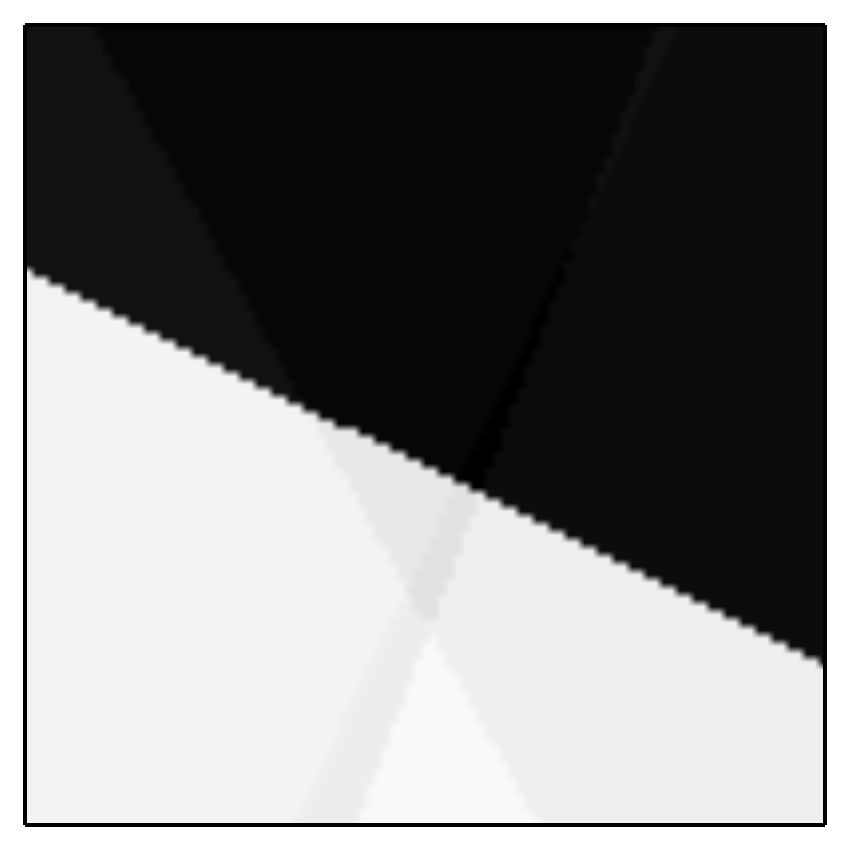}
		\caption{$p(z_{n+1}|z_{1}^{n})$}\label{subfig:conditional-full} 
	\end{subfigure}
    \caption{Illustration of the data distributions in Section~\ref{sec:experiments} by their $y$-expectations. (a): limit distribution of the hidden Markov chain.
    (b)--(d): conditional distributions of the $(n\!+\!1)$-th sample with history lengths 1, 2 and 4. (e):
     conditional distribution of the $(n\!+\!1)$-th sample with full history. }\label{fig:distributions}
\end{figure*}

We illustrate the proposed kernel-based estimator of the conditional risk,
for $d=4$ and a stratified set kernel, $\nK(S,\bar S)=\frac{1}{2|S_+||\bar S_+|}\sum_{(i,j)\in S_+\times\bar S_+}k(x_i,\bar x_j) + \frac{1}{2|S_-||\bar S_-|}\sum_{(i,j)\in S_-\times\bar S_-}k(x_i,\bar x_j)$,
for $S=\{(x_l,y_l)\}_{l=1}^d$ and $\bar S=\{(\bar x_l,\bar y_l\}_{l=1}^d$,
where $S_+/S_-$ are the sets of positive/negative examples in $S$ (and analogously for $\bar S$). 
As base kernel, $k(x,\bar x)$ we use the squared exponential kernel. 
Figure~\ref{subfig:weights} depicts the same samples, $z_1^{N}$, as Figure~\ref{subfig:process},
but each point $z_j$ is drawn at a size proportionally to its contribution to the conditional risk estimate 
at time $N=1001$ for $d=4$, \ie its kernel weight $\nK(z_{j-d}^{j-1}, z^{N}_{N+1-d})$,
The points of the history $z_{N-3}^{N}$ are marked with crosses. 
One can see that the resulting distribution indeed resembles the conditional 
distribution with history length 4 (Figure~\ref{subfig:conditional-h4}),
which is also close to the full conditional distribution (Figure~\ref{subfig:conditional-full}).

\begin{figure*}[t]\centering
\includegraphics[width=250px]{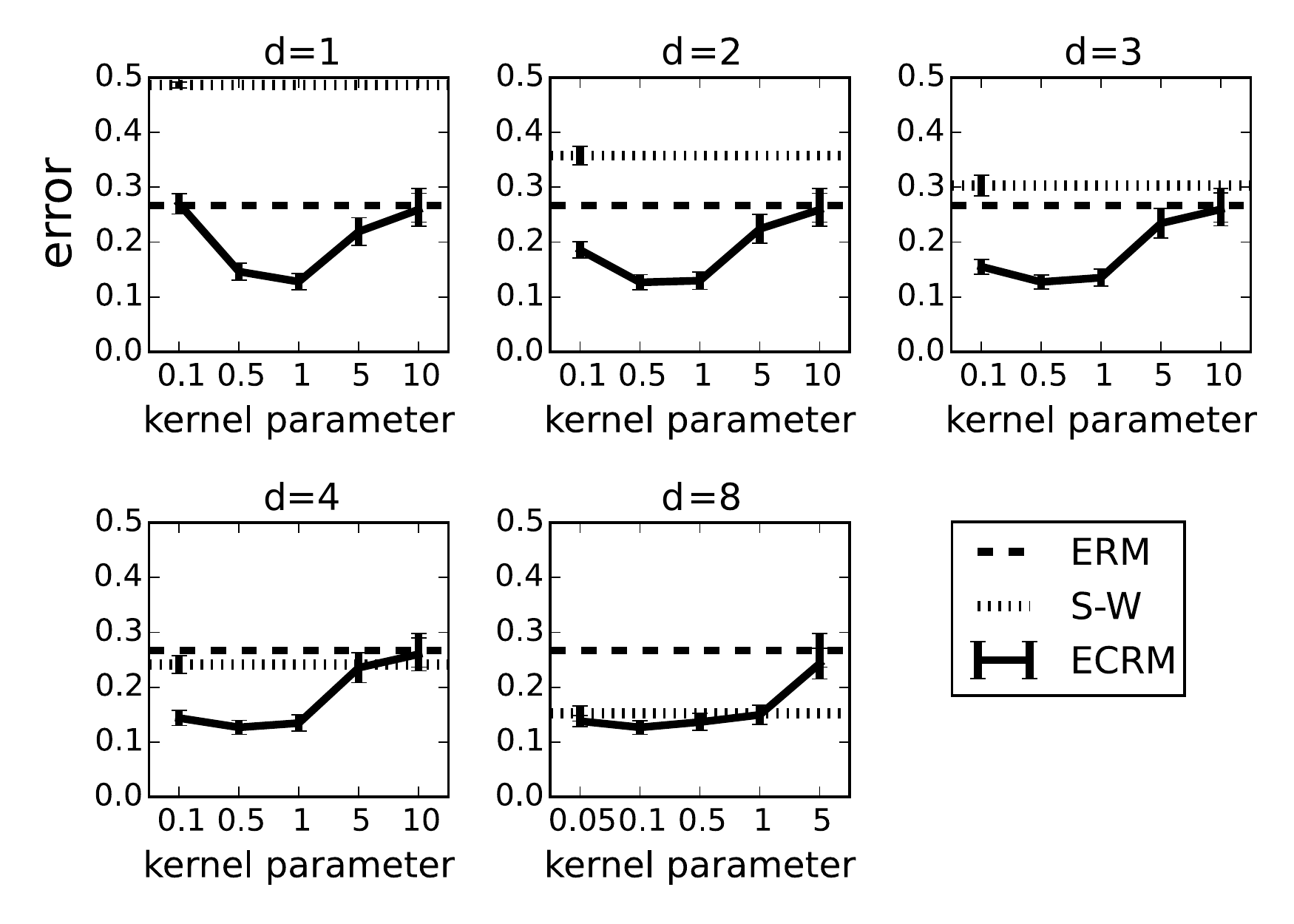}
\caption{Experimental results of empirical conditional risk minimization (ECRM) versus ordinary empirical risk minimization (ERM) and the sliding-window (S-W) heuristic.}
\label{fig:results}
\end{figure*}

Figure~\ref{fig:results} shows a numeric evaluation over $100$ randomly 
generated Markov chains. 
In each case, we learn a linear predictor by approximately minimizing 
the estimator of the conditional risk from $5000$ samples of the process. Then, 
we measure its quality using $5000$ samples from the conditional 
distribution, $p(z_{5001}|z_{1}^{5000})$. 
To learn the predictor, we use squared-loss as a convex surrogate for 
$0/1$-loss in the conditional risk estimator. The resulting expression
is convex in the weight vector of the predictor and can be minimized 
by solving a weighted least squares problem, where the weight of each 
sample is proportional to the kernel 
similarity between its history and the $d$ last observed elements of the sequence. 

The results show that learning by conditional risk minimization (solid line)
consistently outperforms empirical risk minimization (dashed line) for 
a wide range of kernel parameters (bandwidth of the Gaussian base kernel). 
Interestingly, even $d=1$ is often enough to achieve a noticeable improvement.
As an additional baseline, we add a sliding-window method that performs 
empirical risk minimization, but only on the samples in the history 
(dotted line). 
As expected, this heuristic works better than the plain ERM for long 
enough histories, but it does not achieve the same prediction quality 
as the conditional risk minimizer.

\section{Discussions}\label{sec:discussion}
As we noticed in the introduction, our definition of the conditional risk differs from the one usually considered in the literature because we condition on the last $d$ examples instead of to the whole sample.
The relation between two notion of the conditional risks was thoroughly discussed in \citep{caires2005non} and, in short, if one interested in the full history, then it is still possible to find a minimizer of this risk based on the risk with finite history for a wide range of processes called approximately Markov (for Markov processes the two notions coincide for the correct value of $d$). 
A stationary Gaussian process is an example of such process.

This relation also highlights a different problem - which value of $d$ to choose in practice.
This is a hard question for theory, since it requires to take into account computational costs: for bigger values of $d$ it may take more time to compute the weights. 
In addition, one may like to increase $d$ with the amount of data and while there are partial justifications of this approach \citep{Schafer01}, its consequences are unclear.
The current practical solution is to use cross-validation over feasible choices.

Another important parameter we need to choose is the bandwidth $b$. 
Its choice is a known problem for the kernel estimators themselves.
The usual solution is cross-validation, see \citep{Gyorfi05} for more details.
The proof of validity for such procedure in our setting is possible future work.

With all its benefits, the conditional risk minimization comes with the downside that it requires an intensive computation at each time step.
Usually, the amount of computations for a single kernel is linear in $d$, meaning that it takes $\mathcal{O}(dN)$ to compute the weights, and if $d$ is proportional to $N$, it may become quadratic.
Afterwards, one also has to perform the optimization itself.
While reducing the optimization costs is an algorithm-dependent problem, it may be possible to optimize computation of weights by some iterative procedure exploiting the fact that the finite history $\bz_{N-d+1}^{N}$ changes only by one point at each step.

Our approach inherits the main problem of nonparametric methods -- the curse of dimensionality. 
For high-dimensional spaces this makes it even harder to consider longer history because of the computational consideration above.
We believe that in practice this effect can be mitigated by an appropriate choice of the kernel function.
 
\section{Conclusion and possible extensions}
In this paper we introduced an empirical estimator of the conditional 
risk for vector-valued stochastic processes and we proved concentration 
bounds showing that the estimator converges uniformly to the true risk for large 
classes at an exponential rate, if the process is $\beta$-mixing
with sufficiently fast rates.

It is possible to generalize our results in several ways. 
The stationarity assumption is not essential in our proofs and can be relaxed to conditionally stationary processes from \citep{caires2005non}. In this case \lebesguecont\ and \lipschitzhist\ will have to be assumed for each time step.

The assumption of the support to be a hypercube was also made for convenience and it can easily be any compact set. 
It also can be relaxed by different means.
For example, we could make additional assumptions on the moments of $\numerna$ and $\denumna$ or restrict the supremum over histories to the set $\left\lbrace \bar{z}: \norm{\bar{z}} \leq c_n \right\rbrace$, with $c_n$ increasing as $n$ grows like it was done in \citep{Hansen01}. This would lead to a slower convergence rate, though. 
Another option is to assume that the distribution of $\bz_1^{d+1}$ is tight, meaning that for every $\varepsilon > 0$ there is a compact set $K_\varepsilon$ such that the probability mass assigned to $K_\varepsilon^c$ is less than $\varepsilon$. 
This approach would require the knowledge of behavior of the covering numbers of the sets $K_\varepsilon$.
While all these extension may allow to include more "real life" distributions, it is more of a technical contribution.

While $\beta$-mixing assumption covers a wide range of stochastic processes considered in the literature, there are other dependency measures that may be more suitable for concrete situations.
It should be possible to extend our results to these cases as long as a particular dependency measure allows for uniform convergence of the empirical averages.

%





\renewcommand{\refname}{\normalsize References}
\bibliographystyle{apalike}  
\bibliography{biblio}

\end{document}